\documentclass{article}
\usepackage[utf8]{inputenc}
\usepackage{authblk}

\title{The Elliptical Potential Lemma Revisited}
\author[1]{Alexandra Carpentier}
\author[2]{Claire Vernade}

\author[2]{Yasin~Abbasi-Yadkori}
\affil[1]{Otto Von Guericke Universit\"at Magdeburg}
\affil[2]{DeepMind}
\date{}

\usepackage{natbib}
\usepackage{graphicx}
\usepackage{amsmath, amsthm, amssymb}
\usepackage{dsfont}

\usepackage[textsize=tiny, disable]{todonotes}
\usepackage{tcolorbox}
\usepackage{cleveref}

\newtheorem{lemma}{Lemma}
\newtheorem{proposition}{Proposition}

\newcommand{\tr}{\textsc{Tr}}
\newcommand{\LinUCB}{\textsc{LinUCB}}
\newcommand{\LinRel}{\textsc{LinRel}}
\newcommand{\oful}{\textsc{oful}}
\newcommand{\R}{\mathds{R}}

\begin{document}

\maketitle

\begin{abstract}
    This note proposes a new proof and new perspectives on the so-called Elliptical Potential Lemma. This result is important in online learning, especially for linear stochastic bandits. The original proof of the result, however short and elegant, does not give much flexibility on the type of potentials considered and we believe that this new interpretation can be of interest for future research in this field. 
\end{abstract}

\section{Introduction}

Elliptical potentials are a key tool in the analysis of linear forecasters \footnote{see Chapter 11.7 of \cite{plg2006} for a general overview}. They are a specific kind of time-varying potentials \cite{plg2006} that have nice properties allowing to get simple and elegant bounds for important online learning forecasters such as \emph{ridge regression}. In all generalities, for a vector $u \in \R^d$, any symmetric positive matrix $M\in \R^{d\times d}$ defines an elliptical potential  as
\[
\phi_M(u) = \frac{1}{2} u^\top M u
\]
which admits as dual $\phi^*_M(u)=\frac{1}{2}\sqrt{uM^{-1}u}:=\frac{1}{2}\|u\|_{M^{-1}}$, which is also an elliptical potential.

In this short note, we study the properties of elliptical potentials of the form
\[
\phi_{M,p}(u) = \frac{1}{2} u^\top M^p u
\]
for some $p>0$. 
In particular, we prove an equivalent of the Elliptical Potential Lemma (see Lemma 11.11 and Theorem 11.7 in \cite{plg2006} and Theorem 19.4 in \cite{lattimore2020bandit}) that is a key component of the regret analysis of algorithm such as \LinUCB (or \LinRel~\cite{auer2002using}, \oful~\cite{abbasi2011improved}).

Our proof in Section~\ref{sec:proof} follows a different path from existing ones and provides a linear algebra insight on this long-known result. Note that the analysis of \LinRel~ by \cite{auer2002using} also relies on similar argument controlling increments of the singular values of the elliptic potential matrix, see Lemma 13 therein, though our approach allows for slightly more general conclusions.

\section{Setting and problem}
Let $u_1,\ldots,u_T$ be a sequence of arbitrary vectors in $\mathbb{R}^d$ such that $\|u_i\|_2 \leq 1$. For any $1\leq t\leq T$, we define
\[
V_t = \sum_{s=1}^{t-1} u_s u_s^\top +\lambda I\, .
\]

The matrix $V_t$ is symmetric positive definite and we denote $\lambda_1(t) \geq \lambda_2(t) \geq ... \geq \lambda_d(t) \geq 0$ its eigenvalues in decreasing order. Namely, we have $V_t = P_t \Sigma_t P_t^\top $ where $\Sigma_t = \text{diag}(\lambda_1(t),\ldots,\lambda_d(t))$ and $P_t$ is a rotation (orthogonal) matrix. 

The sequence of matrices $V_1,\ldots, V_t$ induces time-varying elliptical potential functions
\[
\forall x \in \mathds{R}^d, \quad \Phi_t (x) = x^\top V_t x = \|x\|_{V_t}^2= \|V_t^{1/2}x\|^2_2.
\]
Note also that for all $t\geq t$, $V_t$ is full rank and so $V_t^{-1}$ exists and the dual potential 
\[
\forall x \in \mathds{R}^d, \quad \Phi^*_t (x) = x^\top V^{-1}_t x = \|x\|_{V^{-1}_t}^2= \|V_t^{-1/2}x\|^2_2
\]
is also elliptic. 

In general, for any $p\in \mathbb R^+$, the matrices $V^p_t$ and $V^{-p}_t$ induce elliptical potentials as above.

We are interested in bounding from above sums of the form 
\begin{equation}
\label{eq:V_t}
   \sum_{t=1}^T \sqrt{u_t^\top V_{t}^{-p}u_t} =\sum_{t=1}^T \|u_t\|_{V_{t}^{-p}} \;. 
\end{equation}

In the rest of this document, we actually focus on bounding a slightly different sum, that is 
\begin{equation}
\label{eq:V_t+1}
  \sum_{t=1}^T \sqrt{u_t^\top V_{t+1}^{-p}u_t} =\sum_{t=1}^T \|u_t\|_{V_{t+1}^{-p}},   
\end{equation}

but the next lemma gives a simple bound that implies that a bound on \eqref{eq:V_t+1} is a bound on \eqref{eq:V_t}.

The following lemma relates $\sum_{t=1}^T \|u_t\|_{V_{t+1}^{-p}}$ with $\sum_{t=1}^T \|u_t\|_{V_{t}^{-p}} \leq 2^{p/2}\sum_{t=1}^T \|u_t\|_{V_{t+1}^{-p}}$, which is the relevant quantity in linear bandits.
\begin{lemma}
Assume that $\lambda \geq 1$. It holds that
\begin{equation}
\label{eq:V_t+1_to_V_t}
  \sum_{t=1}^T \|u_t\|_{V_{t+1}^{-p}} \leq \sum_{t=1}^T \|u_t\|_{V_{t}^{-p}} \leq 2^{p/2}\sum_{t=1}^T \|u_t\|_{V_{t+1}^{-p}}.  
\end{equation}

\end{lemma}
\begin{proof}
The LHS of Eq.~\eqref{eq:V_t+1_to_V_t} follows, since $V_{t+1} \geq V_{t}$ by construction of $V_t$, i.e.
\[
\sum_{t=1}^T \|u_t\|_{V_{t+1}^{-p}} \leq \sum_{t=1}^T \|u_t\|_{V_{t}^{-p}}.
\]
Now, regarding the rightmost inequality in Eq.~\eqref{eq:V_t+1_to_V_t}: since $\lambda \geq 1$ and since $\|u_t\|_2 \leq 1$, we have that $u_t u_t^\top \leq \lambda I$. And so 
$$V_{t+1} = V_t+ u_t u_t^\top \leq V_t + \lambda I \leq 2V_t,\,$$
since $V_t \geq \lambda I$. This implies that $V_{t+1} \leq 2V_{t}$, which implies the RHS of Eq.~\eqref{eq:V_t+1_to_V_t}, namely
\[
 \sum_{t=1}^T \|u_t\|_{V_{t}^{-p}} \leq 2^{p/2}\sum_{t=1}^T \|u_t\|_{V_{t+1}^{-p}}.
\]
\end{proof}

\subsection{Main result}

Our main result is the following proposition that generalizes the so-called Elliptical Potential Lemma to arbitrary powers of the inverse covariance matrix $V_{t+1}$. 
\begin{proposition}[Generalized Elliptical Potential Lemma]
\label{prop:main}

For any sequence of  vectors $u_1,\ldots,u_T \in \mathbb{R}^{d\times T}$, we have

\begin{align*}
   &\text{if}~~p>1, &\sum_{t=1}^T \|u_t\|_{V_{t+1}^{-p}} 
    & \leq \sqrt{\frac{T d}{\lambda^{p-1} (p-1)}}\\
    &\text{if}~~p=1, & \sum_{t=1}^T \|u_t\|_{V_{t+1}^{-p}} &\leq \sqrt{Td\log\left(\frac{T+d\lambda}{d\lambda}\right)},\\
  & \text{if}~~p<1 &\sum_{t=1}^T \|u_t\|_{V_{t+1}^{-p}} 
     &\leq \sqrt{\frac{d^p}{1-p}T(T+d\lambda)^{1-p}}
\end{align*}

\end{proposition}

We defer the proof of our main result to the last section of this document. The proof of our Generalized Elliptical Lemma relies on two important technical results that we present in the next section.

\subsection{Discussion on the optimality of the Generalized Elliptical Potential Lemma}

One important question is whether our bounds are tight, at least regarding the dependence in $T$. In their recent work, \cite[Section~3]{li2019nearly}, analyse the tightness of the Elliptical Potential Lemma by proving a lower bound on $\sum_t \|u_t\|_{V^{-1}_{t-1}}$ for the case $d=1$ in $\Omega(\sqrt{T\log(T)})$. This result proves that for the usual potential ($p=1$), the known upper bound is tight. We apply a similar technique for the case $d=1, p >1$ to obtain a $\Omega(\sqrt{T})$ lower bound that confirms that for general potentials, our upper bounds are also optimal in $T$.

\begin{lemma}
For any $T\geq 1$ and $p>1$, there exists a sequence $u_1,\ldots, u_T \in [0,1]$ such that if we let $V_0=\lambda>0$ and $V_t=V_{t-1}+u_tu_t^\top = V_{t-1}+u_t^2$, we have
\[
\sum_{t=1}^T \sqrt{\frac{u_t^2}{V_{t-1}^p}} \geq \sqrt{T}(\lambda+1)^{-p/2}
\]
\end{lemma}

\begin{proof}
Let us define for any $t\leq T$,  $u_t = \sqrt{1/T}$. Then $V_t = (\lambda + t/T) \in [\lambda, \lambda+1]$. So that
     \begin{align*}
   \sum_{t=1}^T \|u_t\|_{V_{t+1}^{-p}} = \sum_{t=1}^T \| V^{-p/2}_{t+1} u_t \|_2 
    & \geq \sum_{t=1}^T T^{-1/2} (\lambda +1)^{-p/2} = \sqrt{T} (\lambda +1)^{-p/2}.
    \end{align*}

\end{proof}

\section{Proof of \cref{prop:main}}
\label{sec:proof}

\begin{proof}[Proof of Proposition~\ref{prop:main}]
The key element of this proof is \cref{lem:lambda_incr}, proved in the next section, that links $V_{t+1}^{-p}$ and the eigenvalue increments as follows:

\begin{equation}
    \label{eq:lambda_incr}
    \|u_t\|_{V_{t+1}^{-p}} \leq  \sqrt{\sum_{i=1}^d \frac{\lambda_{i}(t+1)-\lambda_{i}(t)}{\lambda_{i}(t+1)^p}}
\end{equation}

Now, importantly note that $\epsilon_{i,t}^2 = \lambda_{i}(t+1)-\lambda_{i}(t)\geq 0$  is non-negative by Weyl's inequality - See \cref{lem:weyl} - since $V_{t+1} = V_t + u_t u_t^\top$. We take adopt the convention that $\epsilon_{i,t}\geq 0$. We can rewrite Eq.~\eqref{eq:lambda_incr} as 
\begin{align}
   \|u_t\|_{V_{t+1}^{-p}} & \leq  \sqrt{\sum_{i=1}^d \frac{\epsilon_{i,t}^2}{\left(\lambda + \sum_{u=1}^{t}\epsilon_{i,u}^2 \right)^p}}, \label{eq:subst}
   \end{align}
since $\lambda_i(t+1) = \lambda + \sum_{u=1}^{t}\epsilon_{i,u}^2$.

   And so since for any sequence $(a_i)_i \in \mathbb R^{+T}$ we have by Jensen's inequality that $\sqrt{T\sum_{i=1}^T a_i^2} \geq \sum_{i=1}^T a_i$, we have
   \begin{align}
   \sum_{t=1}^T \|u_t\|_{V_{t+1}^{-p}}
   & \leq \sqrt{T \sum_{t=1}^{T} \sum_{i=1}^d \frac{\epsilon_{i,t}^2}{\left(\lambda + \sum_{u=1}^{t}\epsilon_{i,u}^2 \right)^p}}.%
   \label{eq:CS}
 \end{align}  

   Now since for any non-negative sequence $(\tilde \epsilon_t)_{t\leq T}$, $\lambda \in \mathbb R^+$ and any decreasing function $f: \mathbb R^+ \rightarrow \mathbb R^+$ we have $\int_{\lambda}^{\lambda + \sum_{t=1}^T \tilde \epsilon_t} f(x) dx \geq \sum_{t=1}^T \tilde \epsilon_t f\Big(\lambda + \sum_{u=1}^t \tilde \epsilon_u\big) $, we have
     \begin{align}
   \sum_{t=1}^T \|u_t\|_{V_{t+1}^{-p}} 
   & \leq \sqrt{T\sum_{i=1}^d \int_{\lambda}^{\lambda + \sum_{u=1}^{T}\epsilon_{i,u}^2} \frac{1}{x^p} dx} = \sqrt{T\sum_{i=1}^d \int_{\lambda}^{\lambda_{i}(T)} \frac{1}{x^p} dx} .
\end{align}

If $p>1$
     \begin{align}
   \sum_{t=1}^T \|u_t\|_{V_{t+1}^{-p}} 
    & \leq \sqrt{\frac{T d}{\lambda^{p-1} (p-1)}}.
\end{align}
If $p=1$
     \begin{align}
   \sum_{t=1}^T \|u_t\|_{V_{t+1}^{-p}} 
    & \leq \sqrt{T\sum_{i=1}^d \log(\lambda_{i}(T)/\lambda)}\\
    & \leq \sqrt{Td\log\left(\frac{T+d\lambda}{d\lambda}\right)},
\end{align}
since $\sum_i \lambda_{i}(T) \leq T +d\lambda$.

If $p< 1$
     \begin{align}
   \sum_{t=1}^T \|u_t\|_{V_{t+1}^{-p}} 
    & \leq \sqrt{T\sum_{i=1}^d \frac{\lambda_{i}(T)^{1-p}}{1-p} }\\
    & \leq \sqrt{\frac{Td\left( \frac{T+d\lambda}{d}\right)^{1-p}}{1-p}} = \sqrt{\frac{d^p}{1-p}T(T+d\lambda)^{1-p}} %
\end{align}
since $\sum_i \lambda_{i}(T) \leq T+d\lambda$.

\end{proof}

\section{Technical Results}

\begin{lemma}
\label{lem:lambda_incr}
For any bounded sequence of vectors $(u_t)_{t\leq T} \in \mathbb{R}^d$, let $\lambda_i(t)$ denote the $i-$th eigenvector of $V_t=\sum_s u_su_s^\top$. Then for any $p>0$, we have
\[
\|u_t\|_{V_{t+1}^{-p}}^2 \leq \sqrt{\sum_{i=1}^d \frac{\lambda_i(t+1)-\lambda_i(t)}{\lambda_i(t+1)^{p}} }  
\]

\end{lemma}

\begin{proof}%
By definition, $V_{t+1}$ is the incremented covariance matrix:  
\begin{align*}
V_{t+1} = V_t + u_t u_t^\top =  P_{t+1} (\tilde \Sigma_t + \tilde u_t \tilde u_t^\top) P_{t+1}^\top =  P_{t+1} \Sigma_{t+1} P_{t+1}^\top  \, \\ \text{  where  }\,~~~~~~~~
\tilde u_t = P_{t+1}^\top u_{t}~~~~~~\tilde \Sigma_t = P_{t+1}^\top V_t P_{t+1}.
\end{align*}

Note that $\tilde \Sigma_t$
is a rotation of $V_t$, viewed in the orthonormalizing basis of $V_{t+1}$. Moreover, expanding the definition of $\tilde \Sigma_t$, we have
\[
  \tilde \Sigma_t = P_{t+1} V_t P_{t+1}^\top = P_{t+1} P_t \Sigma_t P_t^\top P_{t+1}^\top = R_t \Sigma_t R_t^\top, 
\]
and so $\tilde \Sigma_t$ is also a rotation of the diagonal matrix $\Sigma_t$ by the rotation matrix $R_t = P_{t+1} P_t$.

We use properties of the trace norm to lower bound the LHS. 
\[
\sum_{i=1}^d \frac{\lambda_i(t+1)-\lambda_i(t)}{\lambda_i^{p}(t+1)} := \tr\left( \Sigma_{t+1}^{-p} \Sigma_{t+1}\right) - \tr\left(\Sigma_{t+1}^{-p} \Sigma_t  \right) \geq \tr\left( \Sigma_{t+1}^{-p} \Sigma_{t+1} \right) - \tr\left(\Sigma_{t+1}^{-p} \tilde \Sigma_t  \right),
\]
where we used that $\tilde \Sigma_t = R_t \Sigma_t R_t^\top$ and applied Lemma~\ref{lem:rota} to ($\Sigma_{t+1}^p, \Sigma_t$) to prove that $\tr\left(\Sigma_{t+1}^{-p} \tilde \Sigma_t  \right) \geq \tr\left(\Sigma_{t+1}^{-p}  \Sigma_t  \right)$.

Now, by linearity of the trace, we have
\begin{align*}
\tr(\Sigma_{t+1}^{-p} \Sigma_{t+1}) - \tr(\Sigma_{t+1}^{-p} \tilde \Sigma_t)
&= \tr(\Sigma_{t+1}^{-p} (\Sigma_{t+1} - \tilde \Sigma_t) )\\
&= \tr\left(\Sigma_{t+1}^{-p} \tilde u_t \tilde u_t ^\top \right) = \sum_{i \leq d} \frac{\tilde u_i^2}{\lambda_i^p(t+1)}\\
&=\|  \tilde u_t \|^2_{\Sigma_{t+1}^{-p}} = \|u_t\|_{V_{t+1}^{-p}}^2,
\end{align*}
since $\tilde u_t = P_{t+1}^\top u_{t}$ and $\Sigma_{t+1}  = P_{t+1} V_{t+1} P_{t+1}^\top$. This concludes the proof.
\end{proof}

For completeness, we also report some important linear algebra lemmas that we used in our proof.

We start with a standard corollary of Von Neumann's trace inequality, see~\cite[page 340-341]{marshall1979inequalities}.
\begin{lemma}\label{lem:rota}
 Let $\mathcal{O}_d$ is the set of all rotation matrices. For any positive diagonal matrices $\Sigma, \Sigma'$, that both have their eigenvalues ordered in decreasing order, we have that for any $R,Q \in \mathcal{O}_d$
 \[
\tr\left( \Sigma^{-1} \Sigma'\right)  \leq \tr\left( Q\Sigma^{-1}Q^\top R\Sigma'R^\top \right) \leq \tr\left( \Sigma^{-1} \Sigma''\right),
\]
where $\Sigma''$ has the same eigenvalues of $\Sigma'$ but ordered in increasing order.

This implies in particular
\[
\tr\left( \Sigma^{-1} \Sigma'\right)  = \inf_{R \in \mathcal{O}_d} \tr\left( \Sigma^{-1} R\Sigma'R^\top \right).
\]

\end{lemma}

Weyl's inequality for perturbations on Hermitian matrix is a standard but clever algebraic lemma \cite{tao-blog, horn1994topics}.
\begin{lemma}[Weyl's inequality]\label{lem:weyl}
Let $A$ and $B$ be two $d$-dimensional Hermitian matrices, then for all $i,j,k,r,s \in \{1, \ldots, d\}$ such that $j+k-n\geq i \geq r+s - 1$, we have 
\[
\lambda_j(A) + \lambda_k(B) \leq \lambda_{i}(A+B) \leq \lambda_r(A) + \lambda_s(B),
\]
where for a d-dimensional Hermitian matrix $M$ we write $\lambda_1(M) \geq \lambda_2(M) \geq \ldots \geq \lambda_d(M)$ for its (real) eigenvalues.

In particular, when $B$ is a (real) positive and symmetric matrix, we have for all $i \in \{1, \ldots, n\}$
\[
\lambda_i(A) \leq \lambda_i (A+B).
\]
\end{lemma}

\section{Conclusions}
We provide this technical result with no direct application in mind. We believe that this linear algebra point of view on the Elliptical Potential Lemma may be helpful to prove other similar inequalities involving unusual potentials. 
It is possible that in some linear bandit models with heavy-tailed and/or censored feedback, such inequality would be required and we hope that this note will help future research on the topic.

\paragraph{Acknowledgements.} The work of A. Carpentier is partially supported by the Deutsche Forschungsgemeinschaft (DFG) Emmy Noether grant MuSyAD (CA 1488/1-1), by the DFG - 314838170, GRK 2297 MathCoRe, by the DFG GRK 2433 DAEDALUS (384950143/GRK2433), by the DFG CRC 1294 'Data Assimilation', Project A03, and by the UFA-DFH through the French-German Doktorandenkolleg CDFA 01-18 and by the UFA-DFH through the French-German Doktorandenkolleg CDFA 01-18 and by the SFI Sachsen-Anhalt for the project RE-BCI.

\bibliographystyle{plainnat}
\bibliography{references}

\end{document}